\documentclass[10pt]{article}

\usepackage[T1]{fontenc}
\usepackage[utf8]{inputenc}
\usepackage{lmodern}
\usepackage{amsmath,amssymb,amsthm}
\usepackage{mathtools}
\usepackage[margin=1in]{geometry}
\usepackage{graphicx}
\usepackage{booktabs}
\usepackage{enumitem}
\usepackage{microtype}
\usepackage[dvipsnames,svgnames,x11names]{xcolor}
\usepackage[hyphens]{url}
\usepackage[unicode]{hyperref}
\usepackage{bookmark}
\hypersetup{
  pdftitle={ModularPhaseNet: Finite-Cyclic Phase Geometry for Computable Semantic Hierarchy, Direction, and Context Consistency in Standard Transformers},
  pdfauthor={Kiyotaka Kasubuchi; Kazuo Fukiya},
  pdflang={en-US},
  colorlinks=true,
  linkcolor={blue},
  filecolor={Maroon},
  citecolor={Blue},
  urlcolor={blue}}

\newtheoremstyle{modplain}{1.0em}{1.0em}{\normalfont}{}{\bfseries}{.}{ }{}
\theoremstyle{modplain}
\newtheorem{theorem}{Theorem}

\setlist{itemsep=1pt,topsep=3pt,parsep=0pt}
\DeclareMathOperator*{\argmaxop}{arg\,max}
\DeclareMathOperator*{\argminop}{arg\,min}
\newcommand{\Fp}{\mathbb{F}_{p}^{\times}}
\newcommand{\dG}{d_{G}}
\newcommand{\dq}{d_{q}}

\title{ModularPhaseNet: Finite-Cyclic Phase Geometry for Computable
Semantic Hierarchy, Direction, and Context Consistency in Standard
Transformers}
\author{Kiyotaka Kasubuchi \and Kazuo Fukiya}
\date{August 11, 2026}

\begin{document}
\maketitle

\begin{abstract}
We propose ModularPhaseNet, a classical and integer-computable
discretization of the continuous complex phase geometry introduced in
QuantumPhaseNet. The real-valued hidden states of a standard Transformer
are retained, while only an auxiliary phase channel is quantized into a
cyclic subgroup $G=\langle g\rangle\subset\Fp$ of order $q\mid(p-1)$. A
continuous phase $e^{i\phi}$ is represented by $z=g^{a}\bmod p$; phase
composition becomes group multiplication, relative phase becomes group
division, conceptual hierarchy is induced by a filtration of cyclic
quotients, semantic direction is represented by oriented relative group
elements, and contextual consistency is measured by gauge-invariant
cycle holonomy. The method introduces three components into an otherwise
standard Transformer: a finite-phase encoder, a quotient-filtration
hierarchy module, and a group-valued connection module. Their outputs
enter self-attention as real-valued bias terms. Training uses
distributions in the real group algebra or straight-through
Gumbel-Softmax, whereas inference uses exact modular exponentiation and
precomputed tables. No quantum hardware, complex-valued matrix
multiplication, or discrete-logarithm computation is required. We prove
quantization-distortion bounds, nesting of quotient-induced partitions,
gauge invariance, a discrete integrability result for flat connections,
and boundedness of the resulting attention output. The central empirical
hypothesis is that these exact discrete invariants improve hierarchy
recovery, discourse alignment, contradiction detection, and calibrated
hallucination-risk prediction under a controlled compute budget. This
paper reports the theory together with a pre-registered evaluation plan;
the experiments described in Section~\ref{sec:plan} have not yet been
carried out, and no empirical result is claimed here.
\end{abstract}


\section{Positioning of This Work}\label{sec:positioning}

\subsection{Continuity from WavePhaseNet and QuantumPhaseNet}

WavePhaseNet outlined a way to make a semantic conceptual hierarchy
structure (SCHS) explicit in a form that a Transformer can handle: a
discrete Fourier transform along the sequence direction decomposes the
low frequencies into the global theme and intent of a text, and the high
frequencies into local syntax and surface expression (Kasubuchi and
Fukiya 2026a). QuantumPhaseNet extended that construction to a semantic
manifold, a complex vector bundle, a gauge connection, a covariant phase
rate, a semantic wavelength, a connection Laplacian, quantum spectral
selection, and WavePhase Attention (Kasubuchi and Fukiya 2026b).

In the continuous representation of QuantumPhaseNet, a local semantic
state is written schematically as
\begin{equation}\label{eq:cont-state}
\psi(t)=A(t)\,e^{i\phi(t)}u(t),
\qquad
\Omega_{\gamma}(t)=\operatorname{Im}\bigl\langle\hat{\psi}(t),\,
D_{t}\hat{\psi}(t)\bigr\rangle ,
\end{equation}
and the semantic wavelength is defined by
\begin{equation}\label{eq:cont-wavelength}
\lambda_{\gamma}(t)=\frac{2\pi\,c_{s}(\gamma(t),t)}
{\lvert\Omega_{\gamma}(t)\rvert+\varepsilon_{\omega}} .
\end{equation}
A long wavelength serves as a proxy for a higher-level, more abstract
concept, and phase coherence under parallel transport, discourse
direction, geodesic distance, and evidence support are added to the
attention logit.

The aim of the present paper is to make the phase channel of
\eqref{eq:cont-state} and \eqref{eq:cont-wavelength} finite, so that it
can be embedded in a Transformer running on ordinary GPUs or CPUs
without a quantum Fourier transform (QFT), quantum phase estimation
(QPE), an oracle, or amplitude amplification. The real amplitudes, the
embeddings, the query, key and value projections, and the softmax are
all retained unchanged. What is made finite is not the whole semantic
state but an interpretable auxiliary phase channel.

\subsection{Central Hypotheses}\label{sec:hypotheses}

The central hypotheses are the following.

\begin{enumerate}
\item The role played by a continuous phase can be approximated by an
element of a finite cyclic group together with its relative relations.
\item A concept hierarchy can be represented by the nested partition
generated by the quotient maps of a cyclic group, together with a depth
variable for each concept, in much the same way that modular
exponentiation is used in integer factorization.
\item Discourse direction can be represented by an oriented group
element attached to an ordered pair of tokens or clauses.
\item Contextual consistency can be tested by asking whether the cycle
product of independently predicted edge connections returns to the
identity.
\item Adding these corrections to attention makes hierarchy, direction
and consistency explicitly computable inside the same real-valued
arithmetic that a standard Transformer already uses.
\item If these invariants are combined with evidence support and
semantic uncertainty and the result is calibrated, the empirical risk of
hallucination, in the sense of factual error, contradiction and topical
drift, may be reduced.
\end{enumerate}

\subsection{Scope of Claims}\label{sec:scope}

The claim that this paper makes, and the one it is prepared to defend,
is the following.

\begin{quote}
We have not replaced the entire machinery of the continuous complex
phase of QuantumPhaseNet by classical computation. We have extracted
three verifiable structures---semantic hierarchy, oriented semantic
transition, and contextual cycle consistency---as quotient maps and
group-valued connections on a finite cyclic group, and embedded them
into the attention bias and the calibrated risk of a standard
Transformer.
\end{quote}

Conversely, this first version deliberately avoids the following five
claims.

\begin{enumerate}
\item That making the phase finite by itself necessarily eliminates
hallucination.
\item That modular exponentiation has intrinsically greater
representational power than a plain residue index.
\item That a small index of the number theoretic transform (NTT)
automatically represents a higher-level concept.
\item That amplitude amplification equivalent to quantum computation has
been realized classically.
\item That the method is asymptotically faster than dense attention.
\end{enumerate}

Accordingly, we do not claim that hallucination disappears theoretically
merely because a finite group is used. What we propose is a framework
that introduces inconsistencies of semantic hierarchy, of direction, and
of local and global context into a Transformer as reproducible discrete
invariants, and that empirically calibrates a risk which includes
insufficient evidence and semantic uncertainty. Theorems about internal
consistency and empirical hypotheses that can be refuted on real data
are kept strictly separate throughout.

\subsection{Main Contributions}

The contributions of this paper can be organized into the following
eight points.

\begin{enumerate}
\item We define a finite phase quantization that maps a continuous
complex phase $e^{i\phi}$ to a modular exponential $g^{a}\bmod p$, and
we bound the resulting approximation error with respect to distance on
the circle.
\item We add a depth variable to the quotient filtration of a cyclic
group and thereby define an asymmetric ancestor--descendant relation
between a superordinate and a subordinate concept.
\item We represent a multi-axis semantic hierarchy by several cyclic
group heads.
\item We define a group-valued connection, a covariant residual, and a
cycle holonomy on a token graph.
\item We show a discrete integrability result: if every cycle is flat,
then a global assignment of semantic phases exists.
\item We integrate hierarchy, direction, contextual consistency, and
evidence into real-valued logit corrections of standard attention.
\item We separate the differentiable group-algebra representation used
during training from the exact modular-exponential representation used
at inference time.
\item We separate theoretical internal consistency from the falsifiable
experimental hypothesis concerning the reduction of hallucination.
\end{enumerate}

\section{Relation to Prior Work}\label{sec:related}

The standard Transformer learns dependencies through inner products
(Vaswani et al.~2017). FNet uses the Fourier transform as a token mixer
(Lee-Thorp et al.~2022), and rotary position embedding (RoPE) injects
relative positional dependence into attention by means of rotations (Su
et al.~2023). These are close in spirit to the present work, but the
phase used here carries more than position: it carries a conceptual
address, a hierarchical depth, an oriented semantic transition, and
cycle consistency.

The number theoretic transform (NTT) realizes the same sum--product
structure as the DFT over a finite field, and it can compute cyclic
convolutions without complex numbers (Pedrouzo-Ulloa,
Troncoso-Pastoriza, and P\'erez-Gonz\'alez 2017). A finite field,
however, has no natural order relation, absolute value, or energy
ordering comparable to those of the reals or the complex numbers. One
therefore cannot say automatically that ``a small index $k$ is
semantically low-frequency.'' We treat the NTT as an implementation
device for the cyclic convolution of group distributions or for
consistency computations, and we obtain the grounds for a semantic
hierarchy from the quotient-group filtration and from supervision
instead.

It has already been reported that neural networks that learn modular
arithmetic acquire Fourier-like or block-circulant features (Gromov
2023). The present paper, by contrast, does not solve modular arithmetic
itself; it uses finite group structure as an auxiliary inductive bias
for natural language representation.

Order Embeddings (Vendrov et al.~2016) and Poincar\'e Embeddings (Nickel
and Kiela 2017) are strong baselines for hierarchical representation.
The regular cyclic quotient hierarchy used here does not represent an
arbitrary tree without distortion. We compare against these baselines in
the evaluation plan and consider the possibility of combining multiple
heads, unused codes, and hyperbolic embeddings.

\subsection{Correspondence Between the Continuous and Finite-Group
Formulations}

\begin{figure}[htbp]
\centering
\includegraphics[width=\linewidth]{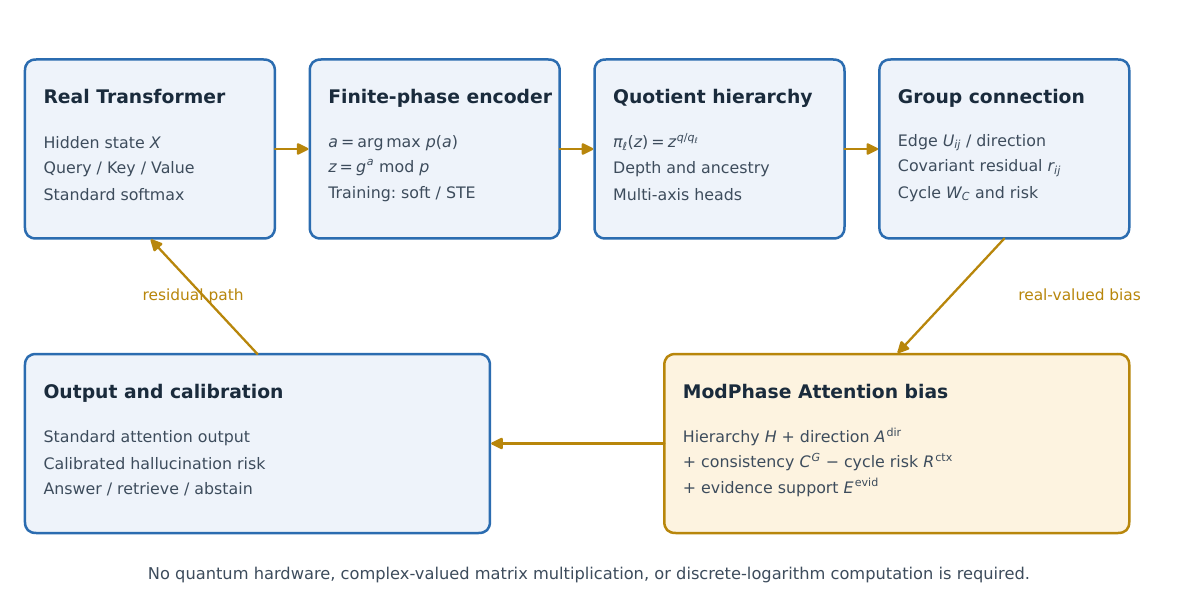}
\caption{Overall structure of ModularPhaseNet. The real-valued
Transformer is retained, and only the results computed in the finite
group channel are returned to attention as a bias and to the calibrated
risk.}
\label{fig:overview}
\end{figure}

Figure~\ref{fig:overview} summarizes the overall structure, and the
correspondence can be organized as follows.

\begin{itemize}
\item \textbf{Complex phase.} $e^{i\phi}$ is mapped to the group phase
$z=g^{a}\bmod p$ and implemented with an integer index and a table of
modular exponentials.
\item \textbf{Phase difference.} This becomes the relative group element
$z_{i}^{-1}z_{j}$, obtained from an index difference or from modular
multiplication and division.
\item \textbf{Semantic wavelength.} This becomes the reciprocal of the
group word distance per step; an integer distance and a real-valued
scale are used together.
\item \textbf{Concept hierarchy.} This is computed from the quotient
maps $\pi_{\ell}\colon G\to G_{\ell}$ together with depth, using
residues, modular exponentials, and comparisons.
\item \textbf{Parallel transport.} The edge connection $U_{ij}\in G$ is
predicted by an edge classifier.
\item \textbf{Curvature and holonomy.} The cycle product
$W_{C}=\prod U_{ij}$ is tested by modular multiplication.
\item \textbf{Discourse direction.} This is represented by a target
group element $D_{S}\in G$ and an oriented group kernel.
\item \textbf{QPE, oracle, and amplitude amplification.} These are
replaced by a learned discrete selector and a risk gate, implemented
with a softmax, a straight-through estimator, and lookups.
\item \textbf{WavePhase Attention.} This becomes ModPhase Attention,
which adds real-valued biases to the standard logit.
\end{itemize}

\section{Finite Cyclic Phase Space}\label{sec:phase-space}

\subsection{Group Parameters}\label{sec:group-params}

Let $p$ be a prime and let $q\mid(p-1)$. Choose an element $g$ of $\Fp$
of order $q$ and set
\begin{equation}\label{eq:group}
G=\langle g\rangle=\{1,g,g^{2},\dots,g^{q-1}\}\subset\Fp ,
\end{equation}
so that $G\cong C_{q}\cong\mathbb{Z}/q\mathbb{Z}$; the existence of such
an element is the standard fact that $\Fp$ is cyclic (Ireland and Rosen
1990). In the implementation we keep the exponent label
$a\in\{0,\dots,q-1\}$ and the group element $z=g^{a}\bmod p$ together as
a pair. Consequently there is never any need to solve the discrete
logarithm $a=\log_{g}z$ at inference time.

The quantization of a continuous phase $\phi\in[0,2\pi)$ is defined by
\begin{equation}\label{eq:quantize}
Q_{q}(\phi)=\operatorname{round}\!\left(\frac{q\phi}{2\pi}\right)
\bmod q,
\qquad
\eta_{q}(\phi)=g^{\,Q_{q}(\phi)}\bmod p ,
\end{equation}
and composition, inversion, and relative phase are computed as
\begin{equation}\label{eq:compose}
z(\phi_{1})z(\phi_{2})\bmod p,
\qquad
z(\phi)^{-1}=z(-\phi),
\qquad
\Delta z_{ij}=z_{i}^{-1}z_{j} .
\end{equation}

\subsection{The Precise Meaning of ``Discrete Phase Geometry''}

Simply endowing a finite set with the usual discrete topology yields an
almost trivial topological structure. What is nontrivial here are the
following three ingredients.

\begin{enumerate}
\item The Cayley graph generated by the generating set
$S=\{g,g^{-1}\}$.
\item The word metric on that Cayley graph.
\item The filtration by subgroups and quotient groups, together with the
group-valued connection on the token graph.
\end{enumerate}

We define the circular distance between exponents by
\begin{equation}\label{eq:word-metric}
\dq(a,b)=\min\{\lvert a-b\rvert,\;q-\lvert a-b\rvert\},
\qquad
\dG\!\left(g^{a},g^{b}\right)=\dq(a,b).
\end{equation}
The normalized distance is $\bar{d}_{G}=2\dG/q\in[0,1]$. Because this
depends on the choice of generator $g$, the model specification fixes
the generator, and the sensitivity to a change of generator is reported
as an ablation.

\subsection{Quantization Distortion}\label{sec:distortion}

Let $d_{S^{1}}(\phi,\psi)\in[0,\pi]$ denote geodesic distance on the
circle.

\begin{theorem}[Distortion of finite phase quantization]
\label{thm:distortion}
For the nearest-neighbor quantization of \eqref{eq:quantize},
\begin{equation}\label{eq:distortion}
\left\lvert\, d_{S^{1}}(\phi,\psi)
-\frac{2\pi}{q}\,\dG\!\left(\eta_{q}(\phi),\eta_{q}(\psi)\right)
\right\rvert\le\frac{2\pi}{q} .
\end{equation}
Moreover, the rounding defect of composition
\begin{equation}\label{eq:defect}
\delta_{Q}=Q_{q}(\phi+\psi)-Q_{q}(\phi)-Q_{q}(\psi)\pmod{q}
\end{equation}
satisfies $\delta_{Q}\in\{-1,0,1\}$ for a suitable choice of
representatives.
\end{theorem}

\begin{proof}[Proof sketch]
The quantization error of each angle is at most $\pi/q$, so by the
triangle inequality the error in the distance between two points is at
most $2\pi/q$. For nearest-neighbor rounding,
$\operatorname{round}(x+y)-\operatorname{round}(x)-\operatorname{round}(y)$
takes one of the values $-1$, $0$, or $1$.
\end{proof}

\subsection{Discrete Semantic Wavelength}

Let $z_{t}=g^{a_{t}}$ be the group phase at time or token step $t$. In
the minimal version, which uses no connection, the discrete phase
velocity is
\begin{equation}\label{eq:phase-velocity}
\nu_{t}=\dG(z_{t},z_{t+1})/\Delta t ,
\end{equation}
and the discrete semantic wavelength corresponding to
\eqref{eq:cont-wavelength} is defined by
\begin{equation}\label{eq:disc-wavelength}
\lambda_{t}^{(q)}=\frac{q\,c_{s}(t)}{\dG(z_{t},z_{t+1})+\varepsilon_{q}} .
\end{equation}
Because $q$ ticks correspond to one full revolution of $2\pi$,
\eqref{eq:disc-wavelength} assigns a longer wavelength to a smaller
phase change. We note explicitly, however, that the word metric changes
when the generator is changed, and that the correspondence between a
long wavelength and a higher-level concept is an empirical hypothesis.

\section{Semantic Concept Hierarchy from a Quotient Filtration}
\label{sec:hierarchy}

\subsection{Nested Quotient Maps}

Choose a sequence of orders satisfying
\begin{equation}\label{eq:filtration}
1=q_{0}\mid q_{1}\mid\cdots\mid q_{L}=q .
\end{equation}
Setting $g_{\ell}=g^{\,q/q_{\ell}}$, the group
$G_{\ell}=\langle g_{\ell}\rangle$ is cyclic of order $q_{\ell}$. The
projections are
\begin{equation}\label{eq:projection}
\pi_{\ell}\colon G\to G_{\ell},
\qquad
\pi_{\ell}(z)=z^{\,q/q_{\ell}}\bmod p ,
\end{equation}
which in the exponent representation reads
$\pi_{\ell}(g^{a})=g_{\ell}^{\,a\bmod q_{\ell}}$.

\begin{theorem}[Nesting of the hierarchical partition]
\label{thm:nesting}
Let $q_{\ell}\mid q_{m}$ with $\ell\le m$. If $\pi_{m}(x)=\pi_{m}(y)$,
then $\pi_{\ell}(x)=\pi_{\ell}(y)$. Consequently the partitions into
equivalence classes generated by $\pi_{0},\pi_{1},\dots,\pi_{L}$ are
monotonically refined from coarse concepts to fine concepts.
\end{theorem}

\begin{proof}
Write $x=g^{a}$ and $y=g^{b}$. The equality $\pi_{m}(x)=\pi_{m}(y)$
means $a\equiv b\pmod{q_{m}}$. Since $q_{\ell}\mid q_{m}$, we obtain
$a\equiv b\pmod{q_{\ell}}$, hence $\pi_{\ell}(x)=\pi_{\ell}(y)$.
\end{proof}

\subsection{Depth and the Asymmetric Ancestor Relation}

The group phase alone can express that two concepts belong to the same
branch, but not which of them is superordinate. We therefore assign to a
concept $c$ both a phase $z_{c}\in G$ and a depth
$\ell_{c}\in\{0,\dots,L\}$, and define the hierarchical address by
\begin{equation}\label{eq:address}
A(c)=\pi_{\ell_{c}}(z_{c})\in G_{\ell_{c}} .
\end{equation}
We then say that a concept $u$ is an ancestor of $v$, written
$u\preceq v$, when
\begin{equation}\label{eq:ancestor}
u\preceq v
\iff
\ell_{u}\le\ell_{v}
\quad\text{and}\quad
\pi_{\ell_{u}}(z_{v})=A(u).
\end{equation}
Because depth and the quotient address are used together, the relation
is asymmetric.

\subsection{Soft Hierarchical Representation}

During training, each token or concept $i$ emits a phase distribution
$\mathbf{p}_{i}\in\Delta^{q-1}$ and a depth distribution
$\mathbf{r}_{i}\in\Delta^{L}$. Folding to level $\ell$ is defined by
\begin{equation}\label{eq:folding}
(P_{\ell}\mathbf{p}_{i})(s)
=\sum_{a\equiv s\ (\mathrm{mod}\ q_{\ell})}\mathbf{p}_{i}(a),
\qquad s=0,\dots,q_{\ell}-1 ,
\end{equation}
and the soft ancestor score is given by
\begin{equation}\label{eq:soft-ancestor}
H(u\to v)
=\sum_{\ell\le m}\mathbf{r}_{u}(\ell)\,\mathbf{r}_{v}(m)\,
\bigl\langle P_{\ell}\mathbf{p}_{u},\,P_{\ell}\mathbf{p}_{v}
\bigr\rangle\in[0,1].
\end{equation}
In the hard one-hot limit this coincides with the indicator function of
\eqref{eq:ancestor}.

\subsection{Multi-Axis Hierarchy}

Corresponding to the multi-axis wavelengths of QuantumPhaseNet, we let
$H$ heads carry different cyclic groups $G^{(h)}$ and depths
$\ell^{(h)}$. Axes such as \emph{is-a}, \emph{part-of}, causal
abstraction, function, and attribute are assigned to different heads,
and the scores are combined as
\begin{equation}\label{eq:multi-axis}
H^{\mathrm{multi}}_{uv}=\sum_{h=1}^{H}w_{h}H^{(h)}(u\to v),
\qquad
w_{h}\ge0,\quad\sum_{h}w_{h}=1 .
\end{equation}
The meaning of each axis is identified by supervised labels, contrastive
learning, or sparsity regularization.

\section{Semantic Direction and the Finite-Group Gauge Connection}
\label{sec:direction}

\subsection{Independent Edge Connections}

Construct a directed graph $\mathcal{G}=(V,E)$ whose vertices are
tokens, sentences, or propositions. Assign to each vertex a phase
$z_{i}\in G$ and to each edge $(j\to i)$ a connection representing
parallel transport,
\begin{equation}\label{eq:connection}
U_{ij}\in G,\qquad U_{ji}=U_{ij}^{-1}.
\end{equation}
What matters is that $U_{ij}$ is \emph{not} simply defined as
$z_{i}z_{j}^{-1}$ but is predicted independently by an edge encoder. If
one always set $U_{ij}=z_{i}z_{j}^{-1}$, every cycle product would be
identically equal to the identity and the contextual consistency loss
would carry no information.

For a local gauge $h_{i}\in G$ the transformation law is
\begin{equation}\label{eq:gauge}
z_{i}\mapsto z_{i}'=h_{i}z_{i},
\qquad
U_{ij}\mapsto U_{ij}'=h_{i}U_{ij}h_{j}^{-1}.
\end{equation}

\subsection{Covariant Phase Residual}

The covariant residual of the edge $(j\to i)$ is
\begin{equation}\label{eq:residual}
r_{ij}=z_{i}^{-1}U_{ij}z_{j}\in G .
\end{equation}
If $r_{ij}=1$, then the phase obtained by transporting $z_{j}$ to $i$
along the connection agrees with $z_{i}$. The consistency score is
\begin{equation}\label{eq:consistency}
C^{G}_{ij}
=\exp\!\left[-\frac{\dG(r_{ij},1)^{2}}{\sigma_{c}^{2}}\right]\in(0,1].
\end{equation}

\begin{theorem}[Local gauge invariance of the covariant residual]
\label{thm:gauge}
Under \eqref{eq:gauge} we have $r_{ij}'=r_{ij}$, and hence $C^{G}_{ij}$
is invariant.
\end{theorem}

\begin{proof}
Because a cyclic group is abelian,
\[
(h_{i}z_{i})^{-1}\bigl(h_{i}U_{ij}h_{j}^{-1}\bigr)\bigl(h_{j}z_{j}\bigr)
=z_{i}^{-1}U_{ij}z_{j}. \qedhere
\]
\end{proof}

\subsection{Discourse Direction}

From the prompt, from summary supervision, from an evidence graph, or
from a pooling head, we obtain a target direction $D_{S}\in G$ along
which the text should proceed. The alignment between the edge direction
and the target direction is defined by
\begin{equation}\label{eq:direction}
A^{\mathrm{dir}}_{ij}
=1-\frac{2\,\dG(U_{ij},D_{S})}{\lfloor q/2\rfloor}\in[-1,1].
\end{equation}
Since $U_{ij}$ is oriented, in general
$A^{\mathrm{dir}}_{ij}\ne A^{\mathrm{dir}}_{ji}$.

The direction of the text as a whole is obtained not as a single mean
angle but as a group Fr\'echet medoid, or as a distribution, over a set
$E_{S}$ of important edges:
\begin{equation}\label{eq:medoid}
D_{S}=\argminop_{d\in G}\sum_{(i,j)\in E_{S}}\alpha_{ij}\,
\dG(U_{ij},d)^{2}.
\end{equation}
Because the set is finite, this can be computed by exhaustive search or
by histogram aggregation.

\section{Context Consistency and Cycle Holonomy}\label{sec:holonomy}

\subsection{Wilson-Loop Invariants}

For a directed cycle $C=(i_{0},i_{1},\dots,i_{m}=i_{0})$, the holonomy
along $C$ is
\begin{equation}\label{eq:wilson}
W_{C}=\prod_{t=0}^{m-1}U_{i_{t}i_{t+1}}\bmod p .
\end{equation}
Because the cyclic group is abelian, $W_{C}$ is invariant under the
local gauge transformation \eqref{eq:gauge}. The contextual cycle risk
is defined by
\begin{equation}\label{eq:cycle-risk}
R^{\mathrm{cyc}}_{C}=\left(\frac{2\,\dG(W_{C},1)}{q}\right)^{2}.
\end{equation}
A triangular cycle tests, in the smallest possible unit, whether three
local inferences can hold simultaneously at the global level.

\subsection{Discrete Integrability}

\begin{theorem}[Global integrability of a flat connection]
\label{thm:integrability}
Let $\mathcal{G}$ be a connected graph with $U_{ji}=U_{ij}^{-1}$. If
$W_{C}=1$ for every cycle $C$, then there exist vertex phases
$\{\zeta_{i}\in G\}$ such that
\begin{equation}\label{eq:integrable}
U_{ij}=\zeta_{i}\zeta_{j}^{-1}
\end{equation}
holds on every edge. The family $\{\zeta_{i}\}$ is unique up to a common
left multiplication applied to all vertices.
\end{theorem}

\begin{proof}[Proof sketch]
Fix a root $r$ and set $\zeta_{r}=1$; let $\zeta_{i}$ be the product of
the connections along an arbitrary path from the root to $i$. The
difference between two paths forms a cycle, so by $W_{C}=1$ the value
does not depend on the path. Extending by one edge yields
\eqref{eq:integrable}. Any two solutions differ by the same group
element at every vertex, by connectedness.
\end{proof}

This theorem states that if local semantic transitions are not mutually
contradictory on any cycle, they can be integrated into a single global
assignment of semantic phases. It is the finite-cyclic-group counterpart
of the connection, curvature, and cohomological consistency of
QuantumPhaseNet, and it is the core result of the present construction.

\subsection{Soft Cycle Loss}

Let $\mathbf{u}_{e}\in\Delta^{q-1}$ be the learned distribution of the
edge phase, and let $*$ denote cyclic convolution on the group. The
cycle distribution is
\begin{equation}\label{eq:cycle-dist}
\mathbf{u}_{C}=\mathbf{u}_{e_{1}}*\mathbf{u}_{e_{2}}*\cdots
*\mathbf{u}_{e_{m}},
\end{equation}
and the differentiable cycle loss is
\begin{equation}\label{eq:cycle-loss}
\mathcal{L}_{\mathrm{cyc}}
=\sum_{C\in\mathcal{C}}\sum_{a=0}^{q-1}\mathbf{u}_{C}(a)
\left(\frac{2\,\dq(a,0)}{q}\right)^{2}.
\end{equation}
During training this is computed in the real group algebra; at inference
time it is checked with the modular-exponential product of hard group
elements.

\section{ModPhase Attention}\label{sec:attention}

\subsection{Finite-Phase Encoder}

Let $x_{i}^{(l)}\in\mathbb{R}^{d}$ be the real-valued hidden state of
token $i$ at layer $l$ of the Transformer. Each phase head $h$ emits
\begin{equation}\label{eq:phase-head}
s_{i,h}=W_{h}^{\phi}x_{i}+b_{h}^{\phi}\in\mathbb{R}^{q_{h}},
\qquad
\mathbf{p}_{i,h}=\operatorname{softmax}\!\left(s_{i,h}/\tau\right).
\end{equation}
Hard samples during training are drawn with a straight-through
Gumbel-Softmax, while at inference time we use
\begin{equation}\label{eq:hard-phase}
a_{i,h}=\argmaxop_{a}\ \mathbf{p}_{i,h}(a),
\qquad
z_{i,h}=g_{h}^{\,a_{i,h}}\bmod p_{h}.
\end{equation}
If the table of $\operatorname{pow}(g_{h},a,p_{h})$ is precomputed, only
an integer lookup is required at run time.

\subsection{Logit Correction}

Write the standard attention logit as
\begin{equation}\label{eq:base-logit}
\ell^{\mathrm{base}}_{ij}
=\frac{\mathbf{q}_{i}^{\top}\mathbf{k}_{j}}{\sqrt{d_{k}}}+M_{ij},
\end{equation}
where $\mathbf{q}_{i}$ and $\mathbf{k}_{j}$ are the query and key
vectors and $M_{ij}$ is the mask. The ModPhase logit is then defined by
\begin{equation}\label{eq:mp-logit}
\ell^{\mathrm{MP}}_{ij}
=\ell^{\mathrm{base}}_{ij}
+\alpha_{c}C^{G}_{ij}
+\alpha_{h}H(i\to j)
+\alpha_{d}A^{\mathrm{dir}}_{ij}
-\alpha_{r}R^{\mathrm{ctx}}_{ij}
+\alpha_{e}E^{\mathrm{evid}}_{ij}.
\end{equation}
Each coefficient $\alpha_{\bullet}\ge0$ is parameterized through a
softplus. Here $R^{\mathrm{ctx}}_{ij}$ is the mean risk over the short
cycles that contain the edge, and $E^{\mathrm{evid}}_{ij}\in[0,1]$ is
the evidence support supplied by retrieved documents, citations, a
knowledge graph, or a similar source.

The output is formed as usual:
\begin{equation}\label{eq:mp-attention}
A^{\mathrm{MP}}_{ij}
=\frac{\exp\!\left(\ell^{\mathrm{MP}}_{ij}\right)}
{\sum_{j'\,:\,M_{ij'}=0}\exp\!\left(\ell^{\mathrm{MP}}_{ij'}\right)},
\qquad
Y=A^{\mathrm{MP}}V .
\end{equation}
Since every quantity computed in the phase group enters as a
real-valued scalar bias or as a real-valued embedding attached to the
query, key and value paths, no kernel of an existing Transformer
implementation needs to be modified.

\subsection{Invariance and Boundedness}

\begin{theorem}[Local gauge invariance of the ModPhase logit]
\label{thm:logit-invariance}
If the consistency score $C^{G}_{ij}$, the quotient hierarchy score, the
direction score, the cycle risk, and the evidence score are constructed
only from gauge-invariant quantities, then \eqref{eq:mp-logit} is
invariant under the transformation \eqref{eq:gauge}.
\end{theorem}

\begin{theorem}[Boundedness of the output norm]\label{thm:bounded}
If every value vector satisfies $\lVert v_{j}\rVert_{2}\le B_{V}$, then
for every row
\begin{equation}\label{eq:bounded}
\lVert Y_{i}\rVert_{2}\le B_{V}.
\end{equation}
\end{theorem}

\begin{proof}
The softmax weights are nonnegative and sum to one, so $Y_{i}$ is a
convex combination of the value vectors.
\end{proof}

Because the hard quantization map is discontinuous at the cell
boundaries, we do not claim global Lipschitz continuity for the model as
a whole. Within the range of soft distributions with $\tau>0$, bounded
logits, and finite-distance kernels, a local Lipschitz estimate is
available for the relaxed model used during training.

\section{Connection to Hallucination Risk}\label{sec:risk}

The decomposition of QuantumPhaseNet into directional deviation, path
deviation, curvature deviation, insufficient evidence, and semantic
uncertainty (Kasubuchi and Fukiya 2026b) is made concrete in the
finite-group version as
\begin{equation}\label{eq:risk}
\mathcal{R}^{\mathrm{MP}}_{t}
=\beta_{h}\left(1-H_{t}\right)
+\beta_{d}\left(1-A^{\mathrm{dir}}_{t}\right)
+\beta_{c}R^{\mathrm{cyc}}_{t}
+\beta_{e}R^{\mathrm{evid}}_{t}
+\beta_{u}\mathbb{H}(\mathbf{p}_{t}),
\end{equation}
where $\mathbb{H}(\mathbf{p}_{t})$ is the entropy over semantic classes
or over the phase group, not over token strings.

By definition $\mathcal{R}^{\mathrm{MP}}_{t}$ is a score and not a true
error probability. On a calibration set we fit
\begin{equation}\label{eq:calibration}
\widehat{P}\left(\mathrm{error}_{t}=1\mid\mathcal{R}_{t}\right)
=\sigma\!\left(a\mathcal{R}_{t}+b\right)
\end{equation}
or an isotonic regression, and evaluate it with the Brier score, the
expected calibration error (ECE), AUROC, AUPRC, and risk--coverage
curves. When the threshold is exceeded, the system chooses among
answering, re-retrieving, regenerating, and abstaining. If a guarantee
is to be claimed, the assumptions of conformal abstention must be stated
separately (Abbasi-Yadkori et al.~2024).

\section{Training Objective}\label{sec:objective}

The overall loss is
\begin{equation}\label{eq:total-loss}
\mathcal{L}=\mathcal{L}_{\mathrm{LM}}
+\lambda_{\phi}\mathcal{L}_{\phi}
+\lambda_{h}\mathcal{L}_{\mathrm{hier}}
+\lambda_{d}\mathcal{L}_{\mathrm{dir}}
+\lambda_{c}\mathcal{L}_{\mathrm{cyc}}
+\lambda_{g}\mathcal{L}_{\mathrm{gauge}}
+\lambda_{e}\mathcal{L}_{\mathrm{evid}}
+\lambda_{\mathrm{cal}}\mathcal{L}_{\mathrm{cal}}
+\lambda_{\mathrm{ent}}\mathcal{L}_{\mathrm{ent}} .
\end{equation}

\subsection{Phase and Connection Losses}

For positive edges $E^{+}$ and negative edges $E^{-}$ we use
\begin{equation}\label{eq:phase-loss}
\mathcal{L}_{\phi}
=\sum_{(i,j)\in E^{+}}\left(1-C^{G}_{ij}\right)
+\sum_{(i,j)\in E^{-}}\max\!\left(0,\;C^{G}_{ij}-m_{\phi}\right).
\end{equation}
This encourages consistency between the node phases and the
independently predicted edge connections; but since forcing them always
to agree would make the cycle test trivial, $\lambda_{\phi}$ must not be
set too large.

\subsection{Hierarchy Loss}

For supervised ancestor pairs $P^{+}$ and non-ancestor pairs $P^{-}$,
\begin{equation}\label{eq:hier-loss}
\mathcal{L}_{\mathrm{hier}}
=-\sum_{(u,v)\in P^{+}}\log H(u\to v)
-\sum_{(u,v)\in P^{-}}\log\left(1-H(u\to v)\right)
+\lambda_{\mathrm{depth}}\mathcal{L}_{\mathrm{depth}} .
\end{equation}
Reversed negatives, the transitive closure, and synonym equivalence
classes are evaluated separately.

\subsection{Direction, Evidence, and Calibration}

Against a supervised direction $D_{S}^{*}$ we use
\begin{equation}\label{eq:dir-loss}
\mathcal{L}_{\mathrm{dir}}=\frac{2\,\dG(D_{S},D_{S}^{*})}{q} .
\end{equation}
Evidence support is trained with a claim--evidence binary cross entropy,
calibration with a Brier loss, and phase collapse is suppressed with an
entropy floor or a code-usage regularizer.

\section{The Two-Layer Structure of Training and Inference}
\label{sec:two-layer}

Modular exponentiation $g^{a}\bmod p$, the arg-max, and equality
comparison are not differentiable. Leaving this point vague would be a
serious weakness at review time with respect to end-to-end trainability.
We therefore state the two-layer structure explicitly.

\subsection{Training Time}

\begin{enumerate}
\item Phases are kept soft as $\mathbf{p}_{i}(a)$ and depths as
$\mathbf{r}_{i}(\ell)$.
\item The default is a straight-through estimator whose forward pass
uses hard one-hot values and whose backward pass uses the soft
probabilities.
\item Quotient projection is computed with the folding matrix of
\eqref{eq:folding}.
\item Cycle products are computed by cyclic convolution of the group
distributions.
\item The temperature $\tau$ is annealed downward in stages and code
usage is monitored.
\end{enumerate}

\subsection{Inference Time}

\begin{enumerate}
\item Fix $a=\argmaxop \mathbf{p}(a)$.
\item Look up $z=\operatorname{pow}(g,a,p)$.
\item Compute relative phases, quotient projections, and cycle products
with integer arithmetic.
\item Convert the learned group kernels into real-valued biases.
\item Choose among answering, re-retrieving, regenerating, and
abstaining according to the calibrated risk.
\end{enumerate}

\subsection{Limited Use of the NTT}

When the sequence length satisfies $N\mid(p-1)$, an $N$-th primitive
root $\omega$ allows one to compute
\begin{equation}\label{eq:ntt}
\hat{x}_{k}=\sum_{n=0}^{N-1}x_{n}\omega^{kn}\bmod p,
\qquad
x_{n}=N^{-1}\sum_{k=0}^{N-1}\hat{x}_{k}\omega^{-kn}\bmod p .
\end{equation}
The NTT is a candidate implementation for performing the cyclic
convolution of quantized group distributions in $O(q\log q)$; the
underlying harmonic analysis on finite groups is classical (Terras 1999;
Serre 1977). We do not, however, give the coefficients over
$\mathbb{F}_{p}$ the same energy interpretation as complex amplitudes;
low-frequency semantics are verified either by supervision or by a
spectral module on the real-valued side.

\section{Implementation Algorithms}\label{sec:implementation}

\subsection{Minimum Viable Implementation}

The recommended initial setting is $p=257$, $q=256$, $g=3$. Here $3$ is
an element of order $256$ in $\mathbb{F}_{257}^{\times}$, which permits
$q_{\ell}=2^{\ell}$ for $\ell=0,\dots,8$. In a first experiment intended
to keep the training cost low, one may start from $q=32$ or $q=64$.

\begin{enumerate}
\item Insert phase adapters into the top two to four layers of a
pretrained Transformer.
\item Emit phase logits and depth logits for each token.
\item Predict edge connections for a local window, for claim--evidence
edges, and for inter-sentence edges.
\item Sample only triangular cycles and short fundamental cycles.
\item Add the correction of \eqref{eq:mp-logit} to the attention logit.
\item Freeze the backbone at first and train only the adapters, then
apply LoRA or full fine-tuning.
\item Train the risk calibrator with the training, calibration, and
final evaluation splits kept separate.
\end{enumerate}

\subsection{Scalable Implementation}

A full soft distribution requires $O(Nq)$ memory together with pairwise
computation. For long texts we use the following.

\begin{enumerate}
\item Hard phases with straight-through gradients.
\item Precomputed quotient codes and integer comparison.
\item Restriction of the cycle loss to a local window or to the edges of
the retrieval graph.
\item Restriction of the hierarchy bias to sparse blocks or to the
top-$k$ attention candidates.
\item Splitting into several small group heads
($q_{h}=16,32,64$).
\item Tabulation of exponentials, inverses, projections, and kernels as
lookup tables.
\end{enumerate}

\section{Computational Complexity}\label{sec:complexity}

Standard dense attention takes time $O(N^{2}d)$ and $O(N^{2})$ memory
for the attention matrix. ModularPhaseNet does not improve this
complexity automatically. For hard-phase inference with $H$ heads and
$L$ hierarchy levels, the additional cost is roughly
\begin{equation}\label{eq:complexity}
T_{\mathrm{MP}}
=O\!\left(NHd_{\phi}\right)
+O(NHL)
+O\!\left(N^{2}H\right)
+O\!\left(\lvert\mathcal{C}\rvert H\right).
\end{equation}
The pair bias term $O(N^{2}H)$ is of lower order than the $O(N^{2}d)$ of
existing attention whenever $H\ll d$. Soft training adds $O(NHq)$ for
the phase logits and up to $O(N^{2}Hq)$ for a dense pair distribution,
so the minimum viable implementation uses a small $q$ or a local window.

Computing the modular exponential every time would require $O(\log q)$
modular multiplications per call, but since the whole range of exponent
labels is small, the implementation precomputes a table of $q$ entries
and performs an $O(1)$ lookup. The value here lies not in acceleration
but in the inductive bias provided by exact discrete composition,
reproducibility, quotient projection, and cycle checking.

\section{Numerical Examples}\label{sec:numerical}

\subsection{Hierarchical Addresses}

Let $p=17$, $q=16$, and $g=3$. The order of $g$ is $16$, and we take
$q_{1}=2$, $q_{2}=4$, $q_{3}=8$, $q_{4}=16$. Fine exponents and their
projections onto each quotient level are assigned as follows.

\begin{itemize}
\item \textbf{dog}: $a=5$, $z=5$; the projections for $q=2,4,8,16$ are
$(16,13,8,5)$.
\item \textbf{wolf}: $a=13$, $z=12$; the projections are
$(16,13,8,12)$.
\item \textbf{cat}: $a=9$, $z=14$; the projections are $(16,13,9,14)$.
\item \textbf{lion}: $a=1$, $z=3$; the projections are $(16,13,9,3)$.
\end{itemize}

At $q=2$ the four concepts fall into the same \emph{animal} class; at
$q=4$ they fall into the same \emph{mammal} class; at $q=8$ they split
into \{dog, wolf\} and \{cat, lion\}; and at $q=16$ they become
individual concepts. If the depth of each concept is recorded alongside
the address, the asymmetric ancestor relation animal $\preceq$ mammal
$\preceq$ canine $\preceq$ dog is expressed by \eqref{eq:ancestor}.

\subsection{Cycle Consistency}

For a triangular cycle, take
\begin{equation}\label{eq:cycle-example}
U_{12}=3^{2}=9,\qquad U_{23}=3^{3}=10,\qquad U_{31}=3^{11}=7
\pmod{17},
\end{equation}
so that
\begin{equation}\label{eq:cycle-ok}
W_{C}=9\cdot10\cdot7\equiv1\pmod{17},
\end{equation}
which is consistent. If the last edge is mispredicted as
$U_{31}=3^{12}=4$, then
\begin{equation}\label{eq:cycle-bad}
W_{C}=9\cdot10\cdot4\equiv3\not\equiv1\pmod{17},
\end{equation}
and one detects that no global phase assignment satisfying all three
local inferences simultaneously exists.

\section{Evaluation Plan}\label{sec:plan}

The evaluation described in this section has not yet been carried out.
We state it here as a pre-registered plan, so that the research
questions, datasets, baselines, ablations, and statistical design are
fixed before any result is observed.

\subsection{Research Questions}\label{sec:rq}

\begin{enumerate}
\item \textbf{RQ1 (hierarchy).} Can the quotient-group filtration
recover hypernym--hyponym relations and conceptual depth on HyperLex and
WordNet?
\item \textbf{RQ2 (direction).} Do oriented group elements predict
sentence order, instruction following, and maintenance of the line of
argument better than ordinary cosine similarity?
\item \textbf{RQ3 (consistency).} Does the cycle risk detect multi-hop
contradictions that a local entailment score alone would miss?
\item \textbf{RQ4 (hallucination).} Does \eqref{eq:risk} calibrate error
better than maximum softmax probability, entropy, SelfCheckGPT (Manakul,
Liusie, and Gales 2023), and semantic entropy (Farquhar et al.~2024)?
\item \textbf{RQ5 (discretization).} Does the performance gap to
continuous QuantumPhaseNet narrow as $q$ increases, and at which $q$
does the elbow between accuracy and computational cost occur?
\item \textbf{RQ6 (value of modular exponentiation).} Is there any
difference among plain angle binning, a learned categorical code, and a
modular-exponentiation checksum?
\end{enumerate}

\subsection{Datasets and Metrics}

\begin{itemize}
\item \textbf{Concept hierarchy.} HyperLex (Vuli\'c et al.~2017) and
WordNet hypernym pairs, measured by Spearman $\rho$, ancestor F1, depth
MAE, and the rate of transitivity violations.
\item \textbf{Contextual consistency.} FEVER (Thorne et al.~2018), NLI,
and synthetic multi-hop cycles, measured by AUROC, macro-F1, and
cycle-risk separation.
\item \textbf{Factuality.} TruthfulQA (Lin, Hilton, and Evans 2022),
FActScore (Min et al.~2023), and LongFact/SAFE (Wei et al.~2024),
measured by factual precision, citation support, and claim recall.
\item \textbf{Calibration.} Brier score, ECE, AUPRC, risk--coverage, and
abstention utility against the error labels above.
\item \textbf{Efficiency.} Tokens per second, peak memory, phase
overhead, and latency, on identical hardware and token counts.
\end{itemize}

\subsection{Comparison Models}

\begin{enumerate}
\item A standard Transformer.
\item A Transformer with RoPE.
\item FNet or a WavePhaseNet module.
\item The classical quantum-inspired version of QuantumPhaseNet.
\item The phase categorical code alone.
\item The hierarchy component alone.
\item The direction component alone.
\item The cycle-consistency component alone.
\item Full ModPhase Attention.
\item A baseline that uses Order Embeddings or Poincar\'e Embeddings as
an auxiliary loss.
\end{enumerate}

\subsection{Required Ablations}

\begin{enumerate}
\item $q\in\{8,16,32,64,128,256\}$.
\item A single group head versus multi-axis group heads.
\item Hard, soft, and straight-through estimators.
\item Edges generated automatically from the node phases versus edges
predicted independently.
\item No cycle loss.
\item No quotient hierarchy.
\item A change of generator, $g\mapsto g^{u}$.
\item Random permutation of the phase labels.
\item Reversal of direction, $U_{ij}\mapsto U_{ij}^{-1}$.
\item Removal of the evidence and calibration terms.
\end{enumerate}

\subsection{Statistical Design}

All comparisons use the same tokenizer, data order, parameter budget,
and number of training tokens. At least five seeds are used, and paired
bootstrap confidence intervals and effect sizes are reported. Multiple
comparisons are corrected by the Holm procedure or an equivalent method,
and additional FLOPs, memory, and latency are reported alongside
performance rather than performance alone. The training set of the risk
calibrator is kept separate from the final evaluation set.

\section{Falsification Criteria}\label{sec:falsification}

If any of the following is observed reproducibly, the corresponding part
of the central hypothesis is rejected or narrowed.

\begin{enumerate}
\item The quotient-group hierarchy fails to outperform a plain
categorical embedding or an Order Embedding.
\item Semantic performance varies substantially with a change of
generator, so that coordinate dependence cannot be controlled.
\item The cycle risk fails to correlate with contradiction labels and
cannot be distinguished from a phase-shuffled control.
\item Increasing $q$ fails to improve the approximation to the
continuous version.
\item The Brier score and ECE of the hallucination risk fail to
outperform an entropy baseline.
\item The improvement is explained entirely by an increase in the number
of parameters or by stronger retrieval.
\item The improvement in factuality per unit of additional computation
is smaller than that of a verifier or reranker under the same compute
budget.
\end{enumerate}

\section{Anticipated Reviewer Comments and Responses}\label{sec:reviews}

\begin{enumerate}
\item \emph{``The topology of a finite group is discrete and therefore
trivial.''} We restrict the claim of novelty: it lies not in the
topology alone but in the integration of the Cayley metric, the quotient
filtration, the group-valued connection, and the cycle holonomy.
\item \emph{``$C_{q}$ and $\mathbb{Z}/q\mathbb{Z}$ are isomorphic, so
modular exponentiation is unnecessary.''} We claim no increase in
representational power. Modular exponentiation is an implementation
representation that provides exact composition, subgroup projection,
NTT compatibility, checksums, and reproducibility. A plain residue
baseline is always included.
\item \emph{``Modular exponentiation is not differentiable.''} We
separate the group-algebra distribution and straight-through estimator
used during training from the hard operations used at inference, and we
ablate the gradient estimator.
\item \emph{``If the cycle product is built from node differences it is
always the identity.''} Edge connections are predicted independently,
and trivial node-derived edges are reported as a negative control.
\item \emph{``A finite field has no natural notion of low frequency or
amplitude.''} We assume no semantic low/high ordering of the NTT index.
Hierarchy is defined from quotient supervision, direction from oriented
edges, and consistency from holonomy.
\item \emph{``A regular cyclic quotient cannot represent an irregular
ontology.''} We use multiple heads, unused codes, depth masks, and
product codes, and we compare against Order and Poincar\'e baselines. No
claim of universal representability is made.
\item \emph{``Reduction of hallucination does not follow from the
mathematics.''} The risk is a pre-calibration score, and the reduction
of factual error is evaluated as an empirical hypothesis.
\end{enumerate}

\section{Limitations}\label{sec:limitations}

\begin{enumerate}
\item \textbf{Representational isomorphism.} The modular-exponential
representation is isomorphic to the exponent-residue representation and
does not by itself create new representational capacity.
\item \textbf{Discretization error.} For small $q$, aliasing at the
phase boundaries and code collisions occur.
\item \textbf{Coordinate dependence.} The word metric and the signed
direction depend on the generator. Reconciling full automorphism
invariance with a meaningful semantic direction requires additional
design.
\item \textbf{Regularity of the hierarchy.} A single chain of cyclic
quotients assumes regular branching. An irregular ontology with multiple
inheritance requires multiple heads or a combination with another
geometry.
\item \textbf{Optimization.} Training can become unstable because of
hard codes, temperature annealing, and codebook collapse.
\item \textbf{Complexity.} A dense pair bias does not remove the
quadratic complexity of attention.
\item \textbf{Semantic identification.} The numerical value of a group
code carries no intrinsic natural-language meaning; it is identified
only through supervision and through the invariants.
\item \textbf{Hallucination.} Group consistency alone cannot explain
every cause, including missing knowledge, retrieval failure, and
citation error. Evidence and calibration remain necessary.
\end{enumerate}

\section{Conclusion}\label{sec:conclusion}

ModularPhaseNet discretizes the continuous complex phase of
QuantumPhaseNet into a finite cyclic group channel that can coexist with
a real-valued Transformer. Concept hierarchy is represented by the
nested partition of quotient groups together with depth, semantic
direction by oriented edge group elements, and contextual consistency by
cycle holonomy. The residual and the cycle product, both invariant under
the local gauge, can be computed exactly with integer modular
exponentiation, and the results are returned to standard attention as
real-valued biases.

The theoretical value of this construction lies in the fact that whether
local semantic transitions can be integrated into a single global
semantic assignment becomes checkable on a finite group. Its practical
value lies in making hierarchy, direction, contradiction, evidence, and
uncertainty observable within one and the same model, without requiring
a complex-valued model or quantum hardware. The reduction of
hallucination is not a theorem but a calibratable empirical hypothesis
whose inputs are exact invariants, and it is to be verified as such.

\section*{References}
\addcontentsline{toc}{section}{References}

\begingroup
\setlength{\parindent}{0pt}
\setlength{\parskip}{6pt}

Abbasi-Yadkori, Yasin, Ilja Kuzborskij, David Stutz, et al.~2024.
``Mitigating LLM Hallucinations via Conformal Abstention.''
arXiv:2405.01563.

Farquhar, Sebastian, Jannik Kossen, Lorenz Kuhn, and Yarin Gal. 2024.
``Detecting Hallucinations in Large Language Models Using Semantic
Entropy.'' \emph{Nature} 630: 625--630.

Gromov, Andrey. 2023. ``Grokking Modular Arithmetic.''
arXiv:2301.02679.

Ireland, Kenneth, and Michael Rosen. 1990. \emph{A Classical
Introduction to Modern Number Theory}. 2nd ed. New York: Springer.

Kasubuchi, Kiyotaka, and Kazuo Fukiya. 2026a. ``WavePhaseNet: A
DFT-Based Method for Constructing Semantic Conceptual Hierarchy
Structures (SCHS).'' arXiv:2602.14419.

Kasubuchi, Kiyotaka, and Kazuo Fukiya. 2026b. ``QuantumPhaseNet: A
Gauge-Covariant Geometric and Quantum-Spectral Theory of Semantic
Concept Hierarchies with Prototype Validation of a Classical
Quantum-Inspired Model.'' arXiv:2608.15820.

Lee-Thorp, James, Joshua Ainslie, Ilya Eckstein, and Santiago
Onta\~n\'on. 2022. ``FNet: Mixing Tokens with Fourier Transforms.''
\emph{NAACL 2022}. arXiv:2105.03824.

Lin, Stephanie, Jacob Hilton, and Owain Evans. 2022. ``TruthfulQA:
Measuring How Models Mimic Human Falsehoods.'' \emph{ACL 2022},
3214--3252. arXiv:2109.07958.

Manakul, Potsawee, Adian Liusie, and Mark J. F. Gales. 2023.
``SelfCheckGPT: Zero-Resource Black-Box Hallucination Detection for
Generative Large Language Models.'' \emph{EMNLP 2023}, 9004--9017.
arXiv:2303.08896.

Min, Sewon, Kalpesh Krishna, Xinxi Lyu, et al.~2023. ``FActScore:
Fine-Grained Atomic Evaluation of Factual Precision in Long Form Text
Generation.'' \emph{EMNLP 2023}. arXiv:2305.14251.

Nickel, Maximilian, and Douwe Kiela. 2017. ``Poincar\'e Embeddings for
Learning Hierarchical Representations.'' \emph{NeurIPS 2017}.
arXiv:1705.08039.

Pedrouzo-Ulloa, Alberto, Juan Ram\'on Troncoso-Pastoriza, and Fernando
P\'erez-Gonz\'alez. 2017. ``Number Theoretic Transforms for Secure
Signal Processing.'' \emph{IEEE Transactions on Information Forensics
and Security}. arXiv:1607.05229.

Serre, Jean-Pierre. 1977. \emph{Linear Representations of Finite
Groups}. New York: Springer.

Su, Jianlin, Yu Lu, Shengfeng Pan, et al.~2023. ``RoFormer: Enhanced
Transformer with Rotary Position Embedding.'' \emph{Neurocomputing} 568.
arXiv:2104.09864.

Terras, Audrey. 1999. \emph{Fourier Analysis on Finite Groups and
Applications}. Cambridge: Cambridge University Press.

Thorne, James, Andreas Vlachos, Christos Christodoulopoulos, and Arpit
Mittal. 2018. ``FEVER: A Large-Scale Dataset for Fact Extraction and
VERification.'' \emph{NAACL 2018}. arXiv:1803.05355.

Vaswani, Ashish, Noam Shazeer, Niki Parmar, et al.~2017. ``Attention Is
All You Need.'' \emph{NeurIPS} 30. arXiv:1706.03762.

Vendrov, Ivan, Ryan Kiros, Sanja Fidler, and Raquel Urtasun. 2016.
``Order-Embeddings of Images and Language.'' \emph{ICLR 2016}.
arXiv:1511.06361.

Vuli\'c, Ivan, Daniela Gerz, Douwe Kiela, Felix Hill, and Anna Korhonen.
2017. ``HyperLex: A Large-Scale Evaluation of Graded Lexical
Entailment.'' \emph{Computational Linguistics} 43 (4): 781--835.
arXiv:1608.02117.

Wei, Jerry, Chengrun Yang, Xinying Song, et al.~2024. ``Long-Form
Factuality in Large Language Models.'' arXiv:2403.18802.

\endgroup

\end{document}